\documentclass[conference]{IEEEtran}
\IEEEoverridecommandlockouts
\usepackage{cite}
\usepackage{amsmath,amssymb,amsfonts}
\usepackage{algorithmicx,algorithm}
\usepackage[noend]{algpseudocode}
\usepackage{graphicx}
\usepackage{textcomp}
\usepackage{xcolor}
\usepackage{url}
\usepackage{amsthm}
\usepackage[misc]{ifsym}

\newtheorem{prop}{Proposition}
\newtheorem{prob}{Problem Statement}

\def\BibTeX{{\rm B\kern-.05em{\sc i\kern-.025em b}\kern-.08em
		T\kern-.1667em\lower.7ex\hbox{E}\kern-.125emX}}
		
\usepackage{tikz}
\usepackage{textcomp}
\usepackage{hyperref}
\newcommand\copyrighttext{%
  \footnotesize \textcopyright 2021 IEEE. Personal use of this material is permitted.
  Permission from IEEE must be obtained for all other uses, in any current or future
  media, including reprinting/republishing this material for advertising or promotional
  purposes, creating new collective works, for resale or redistribution to servers or
  lists, or reuse of any copyrighted component of this work in other works.
  DOI: 10.23919/DATE51398.2021.9473994}
\newcommand\copyrightnotice{%
\begin{tikzpicture}[remember picture,overlay]
\node[anchor=south,yshift=10pt] at (current page.south) {\fbox{\parbox{\dimexpr\textwidth-\fboxsep-\fboxrule\relax}{\copyrighttext}}};
\end{tikzpicture}%
}

\begin{document}
	
	\title{Continuous Safety Verification of Neural Networks
	}

	\author{
		\IEEEauthorblockN{Chih-Hong Cheng\textsuperscript{\Letter}\IEEEauthorrefmark{1}, Rongjie Yan\textsuperscript{\Letter}\IEEEauthorrefmark{2}\IEEEauthorrefmark{3}}\thanks{This research is part of FOCETA project that has received funding from the European Union’s Horizon 2020 research and innovation programme under grant agreement No 956123. This work has been partly funded by Key Research Program of Frontier Sciences, CAS, under Grant No. QYZDJ-SSW-JSC036, the CAS-INRIA major project under No. 171311KYSB20170027.}
		\IEEEauthorblockA{\IEEEauthorrefmark{1}DENSO AUTOMOTIVE Deutschland GmbH, Eching, Germany}
		\IEEEauthorblockA{\IEEEauthorrefmark{2} State Key Laboratory of Computer Science, ISCAS, Beijing, China}
		\IEEEauthorblockA{\IEEEauthorrefmark{3} University of Chinese Academy of Sciences, Beijing, China}
		Email: c.cheng@eu.denso.com, yrj@ios.ac.cn
		
	}
	\vspace{-2mm}
	
	\maketitle
	
	\copyrightnotice
	
	\begin{abstract}
		Deploying deep neural networks (DNNs) as core functions in autonomous driving creates unique verification and validation challenges. In particular, the continuous engineering paradigm of gradually perfecting a DNN-based perception can make the previously established result of safety verification no longer valid. This can occur either due to the newly encountered examples (i.e., input domain enlargement) inside the Operational Design Domain or due to the subsequent parameter fine-tuning activities of a DNN. This paper considers approaches to transfer results established in the previous DNN safety verification problem to the modified problem setting. By considering the reuse of state abstractions, network abstractions, and Lipschitz constants, we develop several sufficient conditions that only require formally analyzing  a small part of the DNN in the new problem. The overall concept is evaluated in a $1/10$-scaled vehicle that equips a DNN controller to determine the visual waypoint from the perceived image.   
		
	\end{abstract}
	
	
	\section{Introduction}

	Deep neural networks (DNNs) have been widely adopted in automated driving, with proven-in-use applications from perception to assisting complex decision making. Deploying autonomous driving functionalities in the open environment creates substantial challenges in the underlying AI (artificial intelligence)/ML (machine learning) component, partly due to the safety-critical nature and partly due to the encountering of ``black swans", i.e., scenarios that were not considered in the design time. Recently developed safety guidelines in automated driving such as ISO TR4804 or UL4600 explicitly suggest the monitoring of abnormal cases in operation, either in an online or an offline fashion. When encountering these cases, the involved DNN component should be improved accordingly. The underlying rationale is to admit the imperfection of the initially engineered system while targeting a \emph{continuous improvement}.
	One natural question that arises is regarding the huge computational efforts in \emph{formally verifying} a DNN: to what extent can results in formal verification of previous models be reused?

	In this paper, we consider the problem of \emph{formal DNN safety verification under continuous engineering (fine-tuning)}. Motivated by concrete issues in automated driving, the problem is related to migrating the result of one verification problem to another, with two problems differing in two aspects.  
	
	\begin{itemize}
		\item Due to the discovery of \emph{black swans}  (more precisely speaking,  out-of-distribution data points) in run time, the \emph{input domain} for verification will be enlarged. Here we take the approach of abstraction-based monitoring~\cite{cheng2019runtime,henzinger2019outside} where the abstraction~$D_{in}$ capturing the in-distribution data will be enlarged to $D_{in} \cup \Delta_{in}$, due to the newly discovered data $d \in \Delta_{in}$ falling outside $D_{in}$. 
		\item  The trained parameters (weights, bias) between two neural networks may only differ slightly, due to their fine-tunings (i.e., the new model is further tuned from the old model with a very small learning rate such as~$10^{-3}$).
		
	\end{itemize}

	To enable continuous verification, one premise is to have reasonable assumptions on proof artifacts reused from the old verification problem to the new one. Concretely, we consider reusing \emph{state abstractions}, \emph{network abstractions}, and \emph{Lipschitz constants}. State abstractions are sound
	over-approximations of all possible reachable states created via layer-wise reasoning methods~\cite{wang2018formal,gehr2018ai2,singh2019abstract,tran2019star}. Network abstraction~\cite{elboher2019an} refers to property-directed abstraction on the structure of a given neural network. Finally, Lipschitz constants are conservative bounds on the degree of output change subject to input change. 
	
	The key contribution of this paper is the establishment of several sufficient conditions utilizing state abstractions, Lipschitz constants, and network abstractions, in order to avoid complete safety verification on the new problem. The satisfaction of sufficient conditions requires checking one or more substantially simpler problems involving local property reasoning over part of the new network. To reuse state abstraction $S_1, S_2, \ldots, S_n$ over intermediate layers, one can take advantage of the precision increase in exact methods or abstraction methods with refinement capabilities to \emph{locally check} if the abstract state~$S_{i-1}$ before layer~$i$, after passing the computation of the layer, can be contained in the abstract state set~$S_{i}$. The locality of subproblems makes the verification pipeline implementable in a parallelized fashion. Reusing Lipschitz constants requires that the worst-case estimate of output change subject to domain enlargement does not influence safety. Lastly, reusing network abstraction requires that the new network can also be transformed into the same abstraction generated from the previous network.
	
	For initial validation, we  use a~$1/10$ scale vehicle platform that performs autonomous lane following using DNN-based visual perception. The incremental verification, due to only checking conditions of substantially simpler problems, leads to a significant performance gain. The required execution time to prove the safety property in the new problem can be as few as~$0.16\%$ of the original problem-solving time, thanks to reusing proof artifacts. Despite encouraging outcomes in our initial evaluation, the research also reveals many interesting yet challenging questions to be further explored.  
	
	The rest of the paper is structured as follows. After highlighting related works in  Section~\ref{sec:related}, Section~\ref{sec:concepts} provides basic definitions of DNN safety verification and defines the variation under continuous engineering. Section~\ref{sec:verification} presents the main technical results on several sufficient conditions for the safety property to hold in the new problem.  We explain our experimental setup in Section~\ref{sec:experiment} and conclude with future work in Section~\ref{sec:conclusion}. Due to space limits, propositions with intuitively simple proofs are omitted. 
	
	\section{Related work}\label{sec:related}
	
	The recent adoption of DNNs in safety-critical applications such as autonomous driving, flight control, and medical diagnostics leads to fruitful results in developing specialized verification and validation techniques for DNNs~\cite{huang2020survey}. For formal verification of safety properties, DNNs using piecewise linear activation functions such as ReLU or Leaky ReLU  are essentially 0-1 integer programs and can be solved with satisfiability modulo theories (SMT)~\cite{katz2017reluplex,ehlers2017formal,katz2019marabou} or mixed-integer linear programming (MILP)~\cite{cheng2017maximum,lomuscio2017approach,dutta2018output} techniques. 
	Such methods could provide sound and complete answers but fail to verify large scale DNNs. 
	To improve the scalability of the solver,  various sound approximation methods such as symbolic interval~\cite{wang2018formal}, zonotope~\cite{gehr2018ai2}, polyhedron~\cite{singh2019abstract}, and star set~\cite{tran2019star} have been adopted to check various safety properties. The approximations provide layered state abstraction and are sound with respect to the given safety properties. 
	To improve the precision of abstraction-based methods, further heuristics are considered to refine specific abstractions (e.g.~\cite{singh2018boosting}). 
	Lipschitz continuity is used in specialized DNN training schemes to intentionally increase the robustness~\cite{tsuzuku2018lipschitz}, while also being adopted for safety verification of DNNs. For example, the set of possible output values can be computed with Lipschitz constant and the given set of inputs, thereby allowing the examination of safety properties~\cite{ruan2018reachability}. Due to the importance of Lipschitz constants in understanding the characteristics of DNNs, several recent results~\cite{fazlyab2019efficient,zou2019lipschitz} focus on the accurate estimation of Lipschitz constants. Additional to the aforementioned methods, the structure of DNNs can also be abstracted~\cite{elboher2019an} such that both exact and approximation methods could be applied. When the false positive happens, refinement over the structure is required.  Lastly, an interesting problem related to ours is a recent work to check the difference of two DNNs~\cite{paulsen2020reludiff}, which provides formal guarantees for the relationship between two networks with forward interval analysis and backward refinement. 
	
	Despite tremendous efforts targeting to formally verify DNNs, to the best of our knowledge, all of the existing results do not consider a continuous verification setup similar to ours. Our key focus is on reusing the proof artifacts in previous verification tasks to avoid re-proving everything from scratch. This is because, from our practical experience, DNNs used in autonomous driving are extremely complex; the required time for completing the verification task can be enormous, and it is a realistic expectation to encounter multiple domain enlargement and fine-tuning activities. To this end, our work well complements all existing works, and results in formal safety verification of DNNs can be  reconsidered under a continuous verification setting.

	\section{Formulation}\label{sec:concepts}

	\subsection{DNN and Safety Verification}
	A feed-forward DNN model consists of an input layer, multiple hidden layers and an output layer. We consider only DNNs after training, i.e., all associated parameters related to the model (weights, bias) are fixed. With fixed parameters, a DNN model is a function $f: X\rightarrow Y$, where $X$ is the input domain, and $Y$ is the output domain. The model $f$ is built out of a sequence of functions $f:=g_n \otimes g_{n-1} \otimes \ldots \otimes  g_2 \otimes
	g_1$, where $n$ is the number of layers in the DNN, and $g_k$ ($k=1, \ldots, n$) is the function transformation in the $k$-th layer. Given an input $x \in X$, the computation of $f$  is $f(x)=g_n(g_{n-1}(\ldots (g_2 (g_1(x)))\ldots))$, i.e., to perform functional composition over $g_1, \ldots, g_n$. The internal transformation $g_k$ can be viewed as another function composition where the input of~$g_k$ is first linearly transformed, followed by a nonlinear computation using functions such as ReLU (Rectified Linear Unit), Leaky ReLU or sigmoid. 
	
	\vspace{1mm}
	
	Given a DNN~$f$, \emph{\textbf{safety verification}} considered in this paper is formulated as follows: Let $D_{in} \subseteq X$ be the set of input values to be verified, and $D_{out} \subseteq Y$ be the set of safe output values. The safety verification problem checks that  given $f$, $D_{in}$ and $D_{out}$, whether the property $\phi^{f}_{D_{in}, D_{out}}$ holds, where
	
	\vspace{-2mm}
	$$\phi^{f}_{D_{in}, D_{out}} := \forall x \in D_{in}: f(x) \in D_{out}$$

	\subsection{DNN Safety Verification under Continuous Engineering}	
	
	We consider the following changes in the process of continuous improvement, allowing us to define the problem of continuous safety verification.

	\begin{itemize}
		\item \textbf{(Domain enlargement)} The domain of input may be enlarged (due to the need to reconsider additional input values). Given $D_{in} \subseteq X$, we use $D_{in}\cup \Delta_{in}\subseteq X$ to represent the new set of input states to be verified against the safety property.
		
		\item \textbf{(Network parameter adjustment)} The parameters of a DNN can be changed via further training. We use the notation $f' := g'_n\otimes g'_{n-1}\otimes  g'_2\otimes g'_1$ to indicate the model being further tuned from $f$. 
		
	\end{itemize}
	
	\begin{prob}
		Given a neural network model~$f$, a modified model $f'$, as well as $D_{in}$, $\Delta_{in}$, and $D_{out}$, the problem of \textbf{Safety Verification between Two Versions (\textbf{SVbTV})} asks the following: Provided that $\phi^{f}_{D_{in}, D_{out}}$ holds, check if $\phi^{f'}_{D_{in}\cup \Delta_{in}, D_{out}}$ holds as well. 
	\end{prob}
	
	\vspace{2mm}
	One sub-case appears when $\Delta_{in} = \emptyset$, i.e., we consider only the change of network parameters. This sub-case is handled in our general framework. Yet the other sub-case appears when $f'$ and $f$ are the same, i.e., we consider the same network under an enlarged input set for safety verification. 
	
	\begin{prob}
		Given a neural network model~$f$  as well as $D_{in}$, $\Delta_{in}$, and $D_{out}$, the problem of \textbf{Safety Verification under Domain Change (\textbf{SVuDC})} asks the following: Provided that $\phi^{f}_{D_{in}, D_{out}}$ holds, check if $\phi^{f}_{D_{in}\cup \Delta_{in}, D_{out}}$ holds as well. 
	\end{prob}
	
	\vspace{1mm}

	\section{Continuous verification}\label{sec:verification}
	
	To address SVbTV or SVuDC problems, it is important to have realistic assumptions in terms of how the proof over the verified property~$\phi^{f}_{D_{in}, D_{out}}$ is stored for reuse. Without any proof reusing, checking the property~$\phi^{f'}_{D_{in}\cup\Delta_{in}, D_{out}}$ needs to completely restart from scratch. In this paper, we assume that the original DNN $f:=g_n \otimes \ldots \otimes 
	g_1$ has been verified to satisfy property~$\phi^{f}_{D_{in}, D_{out}}$, with the proof artifacts available in one or more of the following categories:
	\begin{itemize}
		\item \textbf{Lipschitz constant}  of the original DNN~$f$. Precisely, a Lipschitz constant~$\ell$ is a positive real number such that
		\begin{equation}
		|f(x_1)-f(x_2)|\leq \ell \;|x_1-x_2| ~~\forall x_1,x_2\in X
		\end{equation}
		\item \textbf{State abstractions} $S_1, \ldots, S_n$ over layers to establish the safety proof, where 
		\begin{itemize}
			\item $\forall x \in D_{in}$, $g_1(x) \in S_1$, 
			\item $\forall i \in \{1, \ldots, n-1\}, \forall x_i \in S_i: g_{i+1}(x_i) \in S_{i+1}$, and 
			\item $S_{n} \subseteq D_{out}$. 
		\end{itemize}
		\item \textbf{Network abstraction} $\hat{f}$ of the original DNN~$f$ where 
		$\forall x\in D_{in}$, $f(x)\in \{\hat{f}(x)\;|\;x\in D_{in}\}$, and we use $f\xrightarrow{D_{in}}\hat{f}$  to represent the the relation between~$f$ and~$\hat{f}$. Safety verification utilizing network abstraction techniques is based on establishing the safety proof that $\{\hat{f}(x)\;|\;x\in D_{in}\} \subseteq D_{out}$, where $\hat{f}$ is a structurally simpler network to be verified against.
	\end{itemize}
	
	As stated in previous sections, there are many verification methods to derive Lipschitz constants or to derive various forms of state or network abstractions, and it is beyond the scope of this paper. Recall that our goal is to enable \emph{proof reuse} subject to the verification setting. 
	
	\begin{figure}[t]
		\centering
		\includegraphics[width=0.8\columnwidth]{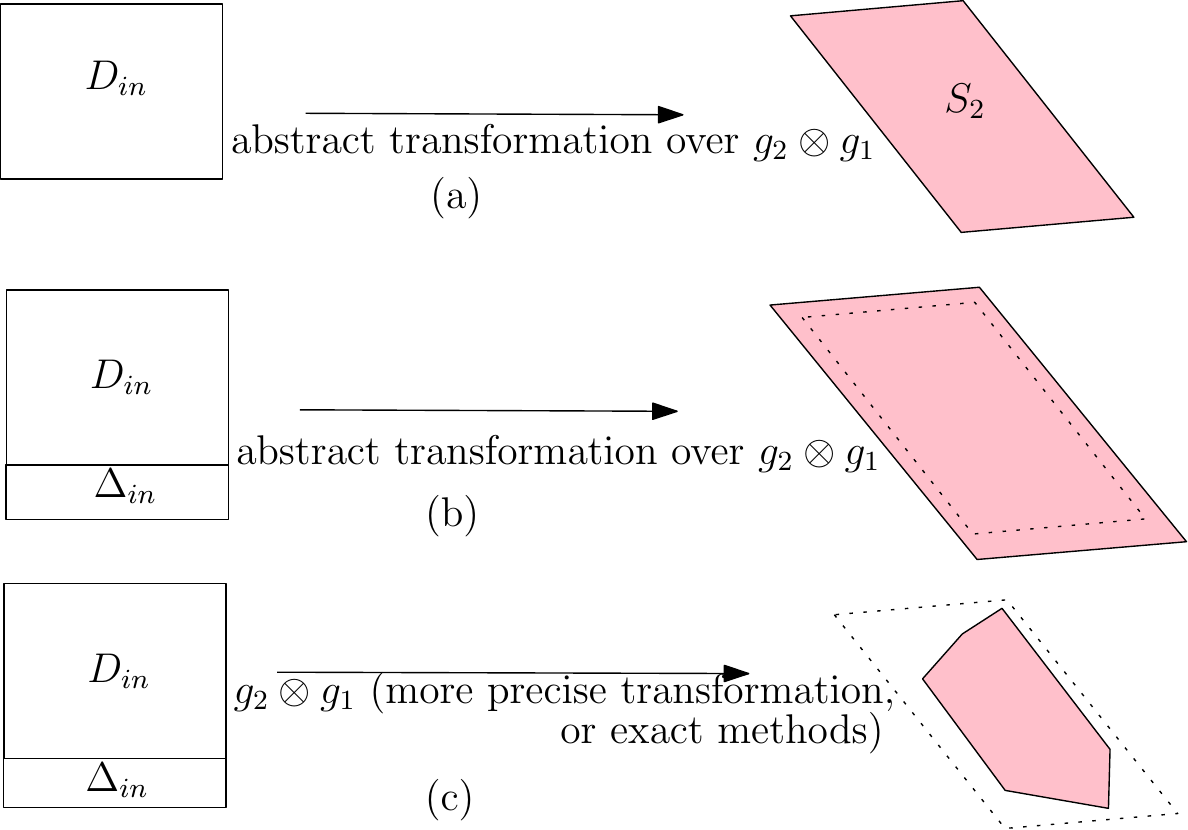}
		\vspace{-2mm}
		\caption{Insight of Proposition~\ref{prop:SVuDC.first.layer}}
		\label{fig.milp.helps}
		\vspace{-3mm}
	\end{figure}
	
	\subsection{Solving SVuDC}
	
	The following results demonstrate that the SVuDC problem can sometimes be solved with a local problem that involves constraint solving of two layers.

	\begin{prop}\label{prop:SVuDC.first.layer}[Proof reuse at layers~1 and 2] Given state abstractions $S_1, \ldots, S_n$ for  establishing the proof of $\phi^{f}_{D_{in}, D_{out}}$. If $\forall x \in D_{in} \cup \Delta_{in}$, $g_2(g_1(x)) \in S_2 $, then property $\phi^{f}_{D_{in}\cup \Delta_{in}, D_{out}}$ also holds. 
	\end{prop}
	
	\begin{proof}  Recall that for the state abstractions $S_1, \ldots, S_n$, they by definition satisfy $\forall i \in \{2, \ldots, n-1\}, \forall x_i \in S_i: g_i(x_i) \in S_{i+1}$ and $S_{n} \subseteq D_{out}$. This means that any state in $S_2$, after passing the rest of DNN, leads to an output that is contained in $D_{out}$. Therefore, so long as $\forall x \in D_{in} \cup \Delta_{in}$, $g_2(g_1(x)) \in S_2$, then $\forall x \in D_{in} \cup \Delta_{in}$, one derives $f(x) \in D_{out}$. 
	\end{proof}
	
	\vspace{-2mm}
	
	To utilize the above proposition, the key is to prove that $\forall x \in D_{in} \cup \Delta_{in}: g_2(g_1(x)) \in S_2$, i.e., all inputs in the enlarged set, after passing $g_2$, can be captured using $S_2$. Illustrated in  Figure~\ref{fig.milp.helps}, using the same abstract transformation over $D_{in} \cup \Delta_{in}$ shall generate an abstract state set  larger than $S_2$ (Figure~\ref{fig.milp.helps}-b), making it impossible to reuse the proof. However, the set of actual reachable values after transformation can be  smaller, similar to the one illustrated in  Figure~\ref{fig.milp.helps}-c. This creates the potential to use methods with higher precision. For example, one
	can use \emph{exact verification methods} that encode $\forall x \in D_{in} \cup \Delta_{in}: g_2(g_1(x)) \in S_2$ as constraints. Note that here the verification method only solves a substantially smaller problem that encodes non-linearity in the first two layers.\footnote{Proposition~\ref{prop:SVuDC.first.layer} involves the computation of two hidden layers rather than one hidden layer; this is based on an observation that existing sound methods using abstract interpretation can lose precision after passing two nonlinear layers, thereby creating space for local solving using exact methods or abstraction-refinement techniques.}
	
	\begin{figure}[t]
		\centering
		\includegraphics[width=0.9\columnwidth]{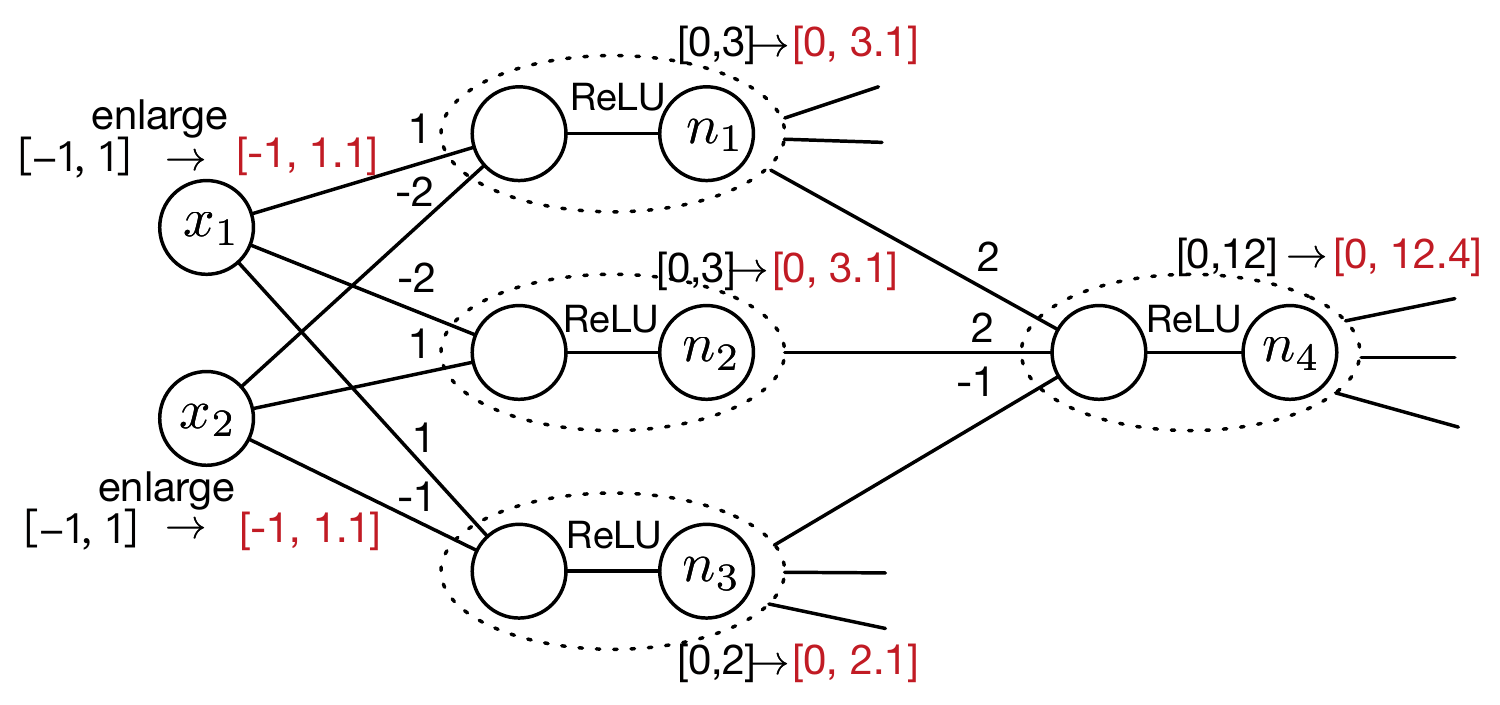}
		\vspace{-2mm}
		\caption{A simple DNN for applying Proposition~\ref{prop:SVuDC.first.layer}}
		\label{fig.prop1.example}
		\vspace{-5mm}
	\end{figure}
	
	\vspace{1mm}
	\noindent{\textbf{(Example)}}  Consider a DNN shown in Figure~\ref{fig.prop1.example} with two inputs and multiple hidden layers. Here we only visualize part of the network to explain the concept. The intervals in black show the intervals of inputs as well as  the result of abstract interpretation\footnote{Here we use boxed abstraction to ease the discussion, but in our evaluation, other types abstract transformers with better precision are used.}; for neuron $n_4$ its value is bounded by~$[0, 12]$. When the input domain is increased from $[-1, 1]\times[-1, 1]$ to $[-1, 1.1]\times[-1, 1.1]$, the result of abstract interpretation  (shown in Figure~\ref{fig.prop1.example} with red texts) indicates that $n_4$ is conservatively bounded by $[0, 12.4]$. If one wishes to reuse the proof, the condition in Proposition~\ref{prop:SVuDC.first.layer} requires that $n_4$ is bounded by $[0, 12]$. The condition can be encoded (Equation~\ref{eq:encode_two_layers}) into a mixed integer programming problem, where the nonlinearity of ReLU can be encoded using big-M approaches~\cite{cheng2017maximum,lomuscio2017approach,dutta2018output}. In this example, exact approaches indicate that the maximum possible value for~$n_4$ equals~$6.2$. As~$6.2 < 12$, the safety property also holds in the enlarged domain.  
	
	\vspace{-5mm}
	\begin{equation}
	\begin{split}
	&-1 \leq x_1 \leq 1.1 \\
	&-1 \leq x_2 \leq 1.1 \\
	&n_1 = {\small\textsf{ReLU}}(x_1 - 2x_2)\\
	&n_2 = {\small\textsf{ReLU}}(-2x_1 + x_2)\\
	&n_3 = {\small\textsf{ReLU}}(x_1 - x_2)\\
	&n_4 = {\small\textsf{ReLU}}(2n_1 +2 n_2 -1 n_3)\\
	&n_4 \geq 12
	\end{split}
	\label{eq:encode_two_layers}
	\end{equation}

	When Proposition~\ref{prop:SVuDC.first.layer} is not applicable, one may proceed with building new sets of abstractions in a layer-wise fashion. During the construction, one can check in parallel if the new state abstraction~$S'_{j}$, after passing to the next layer using exact methods encoding $g_{j+1}$, falls inside the originally created abstraction~$S_{j+1}$. When such a situation occurs, safety is immediately guaranteed in the enlarged domain, as for every $x_{j+1} \in S_{j+1}$, the previous proof artifact ensures that $g_n(g_{n-1}\ldots (g_{j+2}(x_{j+1}))) \in D_{out}$. 
	
	\begin{prop}\label{prop:SVuDC.subsequent.layer}[Proof reuse at layer~$j+1$] Given state abstractions $S_1, \ldots, S_n$ for establishing the proof of $\phi^{f}_{D_{in}, D_{out}}$. The property $\phi^{f}_{D_{in}\cup \Delta_{in}, D_{out}}$ also holds if there exists $j \in \{2, \ldots n-1\}$ and $S'_1, \ldots,  S'_{j-1}$ where
		\begin{itemize}
			\item $\forall x \in D_{in}\cup \Delta_{in}$, $g_1(x) \in S'_1$, 
			\item $\forall i \in \{1, \ldots, j - 1\}, \forall x_i \in S_i': g_{i+1}(x_i) \in S'_{i+1}$, and 
			\item $\forall x_{j} \in  S'_{j}:  g_{j+1}(x_{j}) \in  S_{j+1}$. 
		\end{itemize}
	\end{prop}

	\vspace{1mm}
	Our second line of attacks utilizes the fact that the Lipschitz constant provides a conservative estimation on the new set of output values upon input domain enlargement. This enables another sufficient method to only examine the relation between $S_n$ and~$D_{out}$.

	\begin{prop}\label{prop:SVuDC.last.layer}[Lipschitz-based proof reuse] Given state abstractions $S_1, \ldots, S_n$ for establishing the proof of $\phi^{f}_{D_{in}, D_{out}}$, and given the Lipschitz constant $\ell$ of DNN $f$ over $X$. Let $\kappa$ be a constant where for any input $x_1 \in \Delta_{in}$, its distance to the nearest point $x_2 \in D_{in}$ is bounded by $\kappa$. Then $\phi^{f}_{D_{in} \cup \Delta_{in}, D_{out}}$ holds when the following conditions hold: 
		\[ \forall \hat{s} \in Y, s\in S_n \subseteq Y:  |\hat{s} - s| \leq \ell\kappa \rightarrow  \hat{s} \in D_{out} \]
	\end{prop}

	\begin{proof}
		Given $x_1 \in \Delta_{in}$, following the definition of $\kappa$, there exists $x_2 \in D_{in}$ such that $|x_1 - x_2| \leq \kappa$. Then consider the difference between $f(x_1)$ and $f(x_2)$. The Lipschitz constant~$\ell$ over the complete input domain $X$ ensures the following: 
		\[|f(x_1)-f(x_2)|\leq \ell \;|x_1-x_2| \leq \ell \kappa\]
		
		As $S_n$ contains all $f(x_2)$ where $x_2 \in D_{in}$, the set $\hat{S}_{n}$ by enlarging $S_n$ with all points closer than $\ell \kappa$, i.e.,  $\hat{S}_{n} := \{\hat{s}\; |\; \forall s \in S_n: |\hat{s} - s| \leq \ell\kappa \}$, contains the set  $\{f(x_1) \;|\; x_1 \in \Delta_{in}\}$. So long if any  point of $\hat{S}_{n}$ is also in $D_{out}$ (i.e., the condition in the proposition), the property $\phi^{f}_{D_{in} \cup \Delta_{in}, D_{out}}$ also holds due to the following: 
		\[\{f(x_1) \;|\; x_1 \in \Delta_{in}\} \subseteq \hat{S}_{n} \subseteq D_{out}\]
		
		\vspace{-7mm}
	\end{proof}

	\noindent{\textbf{(Example)}} Consider the input of a DNN being two dimensional. Let $D_{in} = [1,2]\times[1,2]$, and $\Delta_{in} = [0.99,2.01]\times[0.99,2.01] \setminus D_{in}$. Then the smallest value of $\kappa$ can be $\sqrt{0.01^2 + 0.01^2}$, where we set $\kappa$ to be $0.02$ for simplicity purposes. $\kappa$ quantifies the amount of domain enlargement. Let output of the DNN $f$ be one dimensional, and let $D_{out} = [-10,10]$ and $S_n = [1,8]$. Assume that Lipschitz constant $\ell$ equals~$100$, then $\ell\kappa = 2$. Creating by expanding $S_n$ with amount $\ell\kappa$ on both sides, one gets $\hat{S}_n = [1-2, 8+2] = [-1, 10]$. As $ [-1, 10] \subseteq [-10,10]$, the safety property holds not only in $D_{in}$ but also for $D_{in} \cup \Delta_{in}$.

	\subsection{Solving SVbTV}
	
	For the general case where a DNN~$f$ is improved to~$f'$ with a slight change in the underlying parameters, our first strategy is to reuse the previously created state abstraction for~$f$ and check if they can also be used as the state abstraction for~$f'$. 
	
	\begin{prop}\label{prop:SVbTV.all.layers}[Reusing state abstraction - single layer]  Given state abstractions $S_1, \ldots, S_n$ for establishing the proof of $\phi^{f}_{D_{in}, D_{out}}$. When the following conditions hold, the property $\phi^{f'}_{D_{in}\cup \Delta_{in}, D_{out}}$ also holds.
		\begin{itemize}
			\item $\forall x \in D_{in}\cup \Delta_{in}$, $g'_1(x) \in S_1$, 
			\item $\forall i \in \{1, \ldots, n-2\}, \forall x_i \in S_i: g'_{i+1}(x_i) \in S_{i+1}$. 
			\item $\forall x_{n-1} \in S_{n-1}: g'_{n}(x_{n-1}) \in D_{out}$. 
		\end{itemize}
	\end{prop}
	
	Again, checking each condition is a substantially simpler problem that only involves encoding neurons in one layer, and overall there are~$n$ such problems to be checked independently. This makes the checking highly parallelizable and the worst case (under parallelization) is bounded by the maximum number of neurons in one layer. The generalization of Proposition~\ref{prop:SVbTV.all.layers} tries to only select a subset of $S_1, \ldots, S_n$ to be reused, with a price of each subproblem involving multiple layers. Still, each subproblem can be checked independently, thereby utilizing the power of parallelization.  
	
	\begin{prop}\label{prop:SVbTV.all.layers.multiple}[Reusing state abstraction - multiple layers]  Given state abstractions $S_1, \ldots, S_n$ for establishing the proof of $\phi^{f}_{D_{in}, D_{out}}$. The property $\phi^{f'}_{D_{in}\cup \Delta_{in}, D_{out}}$ 
		also holds when  the following condition is met: There exists positive integers $\langle\alpha_{1}\rangle, \langle\alpha_{2}\rangle, \ldots, \langle\alpha_{l}\rangle$, where $1 < \langle\alpha_{1}\rangle < \langle\alpha_{2}\rangle< \ldots < \langle\alpha_{l}\rangle < n-1$, such that
		\begin{itemize}
			\item $\forall x \in D_{in}\cup \Delta_{in}$, $ g'_{\langle\alpha_{1}\rangle}(\ldots(g'_1(x))) \in S_{\langle\alpha_{1}\rangle}$, 
			\item $\forall j \in \{1, \ldots, l - 1\}, \forall x_{\langle{\alpha_{j}\rangle}} \in S_{\langle\alpha_{j}\rangle}:$\\$  g'_{\langle\alpha_{j+1}\rangle}(\ldots (g'_{\langle\alpha_{j}\rangle+1}(x_{\langle \alpha_{j}\rangle}))) \in S_{\langle\alpha_{j+1}\rangle}$. 
			
			\item $ \forall x_{\langle\alpha_{l}\rangle} \in S_{\langle\alpha_{l}\rangle}:  g'_{n}(\ldots (g'_{\langle\alpha_{l}\rangle+1}(x_{\langle\alpha_{l}\rangle}))) \in D_{out}$. 
		\end{itemize}
	\end{prop}
	
	As an example, consider a $6$-layer network, i.e., $n=6$. Let  $l=2$ and $\langle\alpha_{1}\rangle = 2$ and $\langle\alpha_{2}\rangle = 4$. Then the verification is split into three subproblems by reusing~$S_2$ and~$S_4$, namely
	
	\begin{itemize}
		\item $\forall x \in D_{in}\cup \Delta_{in}$, $ g'_{2}(g'_1(x)) \in S_{2}$, 
		\item $\forall x_{2} \in S_{2}:  g'_{4}(g'_{3}(x_{2}))) \in S_{4}$. 
		
		\item $ \forall x_{4} \in S_{4}:  g'_{6} (g'_{5}(x_{4})) \in D_{out}$. 
	\end{itemize}
	
	\vspace{2mm}
	Our final strategy is to reuse the previously created network abstraction for DNN $f$. 
	\begin{prop}\label{prop:SVbTV.network}[Reusing network abstraction] Given network abstraction~$\hat{f}$ of the original DNN~$f$ where $f\xrightarrow{D_{in}}\hat{f}$. Provided that $\phi^{f}_{D_{in}, D_{out}}$ is based on the proof where $\forall x \in D_{in}: \hat{f}(x) \in D_{out}$. Then for the new DNN~$f'$, if $f'\xrightarrow{D_{in}}\hat{f}$, then $\phi^{f'}_{D_{in}, D_{out}}$ also holds.
		
	\end{prop}
	
	\begin{proof}
		If $f'\xrightarrow{D_{in}}\hat{f}$, then based on the definition of network abstraction we have  $\forall x\in D_{in}$, $f'(x)\in \{\hat{f}(x)\;|\;x\in D_{in}\}$.  As in the original proof, we have $\{\hat{f}(x)\;|\;x\in D_{in}\} \subseteq D_{out}$, so  $\forall x\in D_{in}$, $f'(x)\in D_{out}$, thereby implying that $\phi^{f'}_{D_{in}, D_{out}}$ also holds.
	\end{proof}
	
	While Proposition~\ref{prop:SVbTV.network} handles the case with parameter adjustment, encountering both domain enlargement and parameter adjustment, we can first assume that the domain is not enlarged, and perform network transformation to reach the same structure of~$\hat{f}$. Then, we can apply Propositions~\ref{prop:SVuDC.first.layer} or~\ref{prop:SVuDC.last.layer} by reusing the intermediate results in verifying the abstraction of the original DNN $f$ to check whether the property $\phi^{f'}_{D_{in}\cup \Delta_{in}, D_{out}}$ holds also in domain enlargement.

	\subsection{Incremental abstraction fixing for SVbTV}
	
	In this section, we provide a summary on how to proceed when previously mentioned sufficient conditions do not hold. We consider the case over Proposition \ref{prop:SVbTV.all.layers}, where extensions to Proposition \ref{prop:SVbTV.all.layers.multiple} is straightforward. If Proposition \ref{prop:SVbTV.all.layers} does not hold for DNN~$f'$, then one or more of the followings may occur:
	
	\begin{itemize}
		\item $\exists x\in D_{in}\cup \Delta_{in}, g'_1(x)\not\in S_1$, 
		\item $\exists\, i \in \{1, \ldots, n-2\}$ s.t. $\exists x_i \in S_i: g'_{i+1}(x_{i}) \not\in S_{i+1}$, or
		\item $\exists x_{n-1} \in S_{n-1}: g'_{n}(x_{n-1}) \not\in D_{out}$. 
	\end{itemize}

	We further restrict the case where only \emph{one state abstraction}~$S_{i+1}$ is problematic, i.e.,  $\exists x\in S_{i}: g'_{i+1}(x)\not\in S_{i+1}$; for other layers where $j\neq i$, $\forall x \in S_{j}: g'_{j+1}(x_{j})\in S_{j+1}$. We may find a new $S'_{i+1}$ as a replacement of~$S_{i+1}$ to ensure that $\forall x\in S_{i}: g'_{i+1}(x)\in S'_{i+1}$. However, we need 
	a forward propagation from $S'_{i+1}$ to later layers, and check whether there exists $k\in \{i+2,\ldots,n-2\}$, such that $\forall x\in S'_{k}: g'_{k+1}(x)\in S_{k+1}$. The existence of $k$ shows that the propagation from enlarged approximation in earlier layers is again covered by the approximation of later layers in the previous proof. Therefore, the property is preserved. 
	
	When we fail to find a $k$ such that the forward propagation can stop before reaching the last layer, we have to resort to traditional methods to check whether the additional approximation in $S'_{i+1}\setminus S_{i+1}$ is reachable.  The problem is reduced to check the safety property of the sub-network from the new DNN, where the property is encoded with the additional state abstraction. However, in the worst case, the problem can appear in the first state abstraction, and nothing can be reused; this implies that we may need to re-verify the whole network.

	\section{Experiment}\label{sec:experiment}

	\begin{figure}[t]
		\centering
		\includegraphics[width=0.85\columnwidth]{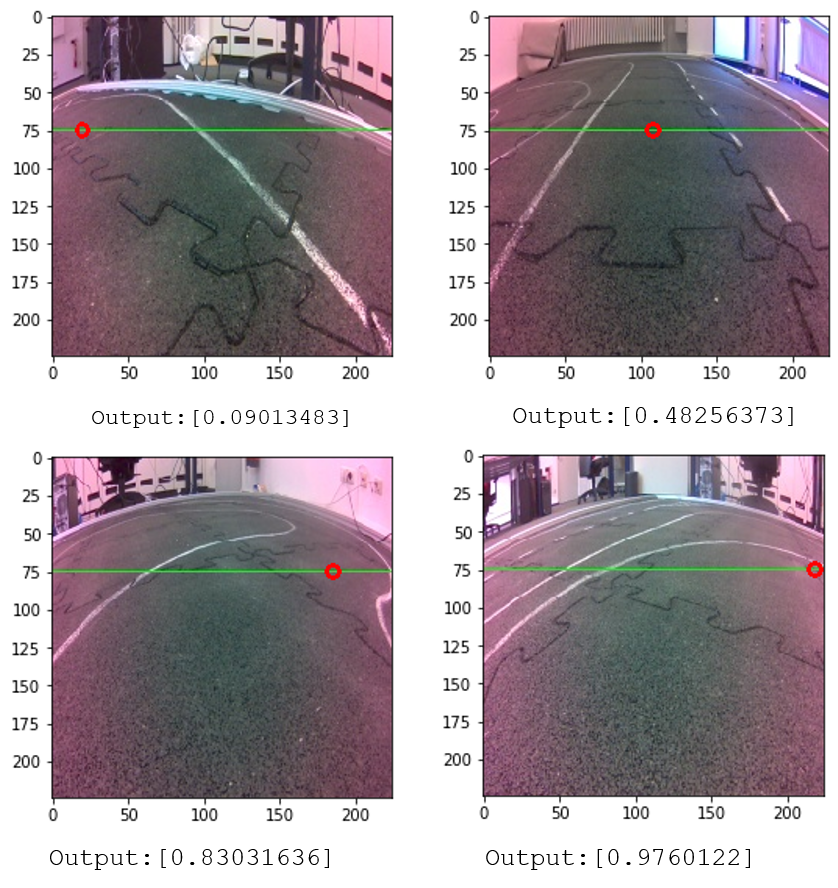}
		\vspace{-2mm}
		\caption{Visualizing the output of the DNN (red circle) on the race track }
		\label{fig.prediction.in.lab}
		\vspace{-2mm}
	\end{figure}
	
	\begin{figure}[t]
		\centering
		\includegraphics[width=0.8\columnwidth]{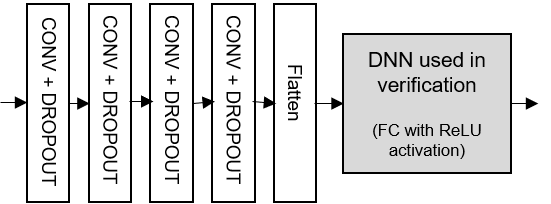}
		\vspace{-2mm}
		\caption{Visualizing the DNN to be formally verified in the experiment}
		\label{fig.DNN.being.verified}
		\vspace{-5mm}
	\end{figure}

	To understand the technology in realistic settings, in our experiment,  we use a~$1/10$ scale vehicle platform that equips with a GPU and a camera module to enable autonomous driving using visual perception. The DNN deployed in the GPU takes an RGB image of size~$224\times224$ and produces a single output~$v_{out}$ within the range~$[0,1]$. Output~$v_{out}$ can be used to reconstruct the visual waypoint via the formula $(x, y) := ({\small\textsf{int}}(224 * v_{out}), 75)$, which suggests the vehicle the next destination on the image plane to follow (examples shown in Figure~\ref{fig.prediction.in.lab}). We first train a DNN on CIFAR10 dataset with an enlarged image matching the dimension,  remove the last layer (i.e., perform transfer learning) and add a new linear layer in order to produce the~$v_{out}$ value. Then the modified single-output network is trained using a manually labeled data set collected on the race track. 
	
	The network to be verified is truncated from the original one for visual perception by taking layers after convolution, as illustrated in Figure~\ref{fig.DNN.being.verified}. The decision is largely due to the limitation of state-of-the-art DNN formal verification tools. The input bound~$D_{in}$ on the verified network is created by recording the minimum and maximum visited neuron value (output of layer ``{\small\textsf{Flatten}}" in Figure~\ref{fig.DNN.being.verified}) for the complete data set, together with additional buffers. Subsequently, we deploy the system, perform continuous execution on the race track, and monitor the DNN. Whenever an image, when being fed into the DNN, creates neuron values (for the output of layer ``{\small\textsf{Flatten}}" in Figure~\ref{fig.DNN.being.verified}) that exceed the bound, the enlarged bound is recorded to form~$D_{in} \cup \Delta_{in}$ for the second verification task. To generate model variations, we fix the weights on the convolution layer and perform fine-tuning such that multiple DNNs to be verified share the same input domain.

	We evaluate the performance of incremental verification in both SVuDC and SVbTV with various models during incremental tuning. 
	To generate state abstractions of layers, we adopt  tool ReluVal~\cite{wang2018formal}, which formally checks properties on a given neural network with symbolic interval analysis.  That is, the state abstraction of a neuron is bounded by its lower and upper valuations. In incremental verification, we set the bounds of state abstractions and check whether there is any violation.  
	To accelerate the verification, we decompose a network into two parts such that
	\begin{itemize}
		\item if the first part preserves the state abstraction, the verification stops in SVuDC case.
		\item verification can be parallelized for the SVbTV case.
	\end{itemize}
	
	Totally we generate four networks from the first in the incremental tuning process. Every network marked with~$i+1$ is obtained by tuning $i$-th network, for $1\leq i\leq 4$. The results are listed in Table~\ref{tab:saving}. As the approximation is usually larger than the reachable states, the verification shows that the properties are preserved in these networks. From Table~\ref{tab:saving}, we observe that incremental verification can always take less than ten percent cost of the original.\footnote{In Table~\ref{tab:saving}, the value for SVbTV is taken by the maximum execution time among all subproblems, as executing of each subproblem checking is independent of others.}  The reasoning is that the scale of the verified network is smaller, and the computation is saved by reusing the established state abstraction. The exception for the first row of SVbTV is that the verification for the first network is fast, making the time saving via decomposition less obvious.

	\begin{table}[t]
		\centering
		\caption{Time savings from incremental verification}
		\vspace{-2mm}
		\begin{tabular}{c|c|c}\hline
			case ID &  SVuDC time / original time & SVbTV time / original time   \\\hline
			1 & 5.27\% & 37.52\% \\ 
			2 & 0.72\% & 4.19\%\\ 
			3 & 0.16\% & 4.68\%\\ 
			4 & 1.34\% & 8.52\%\\ \hline
		\end{tabular}
		\label{tab:saving}
		\vspace{-5mm}
	\end{table}
	
	
	\section{Concluding Remarks}\label{sec:conclusion}
	
	In this paper, we formulated the problem of formal verification of DNN under the continuous improvement setting. We laid the theoretical foundation by proposing multiple sufficient conditions to avoid complete re-verification efforts, and conducted an initial experiment in the lab setting. 
	
	By combining formal verification of DNNs with the concept of continuous engineering, our initial work opens many interesting directions to be explored. Solver-wise, it is worth investigating how exact solvers based on MILP or SMT can be engineered to enable proof reuse, as we observe that techniques such as cuts (conditions to remove real-valued solutions while maintaining  all feasible integer-valued solutions) lose their validity upon domain enlargement. Another direction is to consider the continuous evolution of the quantitative specification of DNN~\cite{seshia2018formal} and the corresponding reuse in the formal verification setting. Yet another direction is to consider the symbolic reasoning using both forward and backward propagation in a continuous verification setup.

\bibliographystyle{IEEEtran}


\end{document}